\documentclass{article}

     \usepackage[final]{CausalDiscoveryWS}
\usepackage[utf8]{inputenc} 
\usepackage[T1]{fontenc}    
\usepackage{hyperref}       
\usepackage{url}            
\usepackage{booktabs}       
\usepackage{amsfonts}       
\usepackage{nicefrac}       
\usepackage{microtype}      

\usepackage{amsmath}
\usepackage{amsthm}
\usepackage{graphicx}

\usepackage{tikz}
\usetikzlibrary{arrows.meta}

\usepackage[ruled,vlined,linesnumbered]{algorithm2e}
\SetKwComment{Comment}{$\triangleright$\ }{}

\def\eqref#1{Equation~\ref{#1}}			
\def\figref#1{Figure~\ref{#1}}			
\def\algref#1{Algorithm~\ref{#1}}		
\def\dffref#1{Definition~\ref{#1}}		
\def\lemref#1{Lemma~\ref{#1}}		

\newtheorem{dff}{Definition}
\newtheorem{thm}{Theorem}
\newtheorem{lem}{Lemma}

\newcommand{\nodes}{\mathbf{V}}
\newcommand{\obs}{\mathbf{O}}
\newcommand{\lat}{\mathbf{L}}
\newcommand{\sel}{\mathbf{S}}
\newcommand{\cdag}{\mathcal{D}}
\newcommand{\cdagset}[3]{\cdag(#1, #2, #3)}
\newcommand{\cdagseti}[4]{\cdag_{#4}(#1, #2, #3)}
\newcommand{\pag}{\mathcal{G}}
\DeclareMathOperator{\MAG}{\mathcal{M}}

\DeclareMathOperator{\DSepop}{\mathbf{D-Sep}}
\DeclareMathOperator{\PDSepop}{\mathbf{Possible-D-Sep}}
\DeclareMathOperator{\PDSeppathop}{\PDSepop_{path}}
\newcommand{\DSEPset}[2]{\DSepop(#1, #2)}
\newcommand{\PDSEPset}[2]{\PDSepop(#1, #2)}
\newcommand{\PDSPath}[3]{\Pi_{#2}(#1, #3)}

\DeclareMathOperator{\Edges}{Edges}
\DeclareMathOperator{\parentsop}{\mathbf{Parents}}
\newcommand{\parents}[2]{\parentsop_{#2}(#1)}

\DeclareMathOperator{\sindep}{Ind}  

\makeatletter
\newcommand*{\indep}{%
  \mathbin{%
    \mathpalette{\@indep}{}%
  }%
}
\newcommand*{\nindep}{%
  \mathbin{
    \mathpalette{\@indep}{\not}
  }%
}
\newcommand*{\@indep}[2]{%
  \sbox0{$#1\perp\m@th$}
  \sbox2{$#1=$}
  \sbox4{$#1\vcenter{}$}
  \rlap{\copy0}
  \dimen@=\dimexpr\ht2-\ht4-.2pt\relax
  \kern\dimen@
  {#2}%
  \kern\dimen@
  \copy0 
} 
\makeatother

\title{A Single Iterative Step for Anytime Causal Discovery}

\author{%
  Raanan Y.~Rohekar
    \\
  Intel Labs\\
  \texttt{raanan.yehezkel@intel.com} \\
  \And
  Yaniv Gurwicz \\
  Intel Labs \\
  \texttt{yaniv.gurwicz@intel.com} \\
  \And
  Shami Nisimov \\
  Intel Labs \\
  \texttt{shami.nisimov@intel.com}
  \And
  Gal Novik \\
  Intel Labs \\
  \texttt{gal.novik@intel.com}\\
}

\begin{document}

\maketitle

\begin{abstract}
  We present a sound and complete algorithm for recovering causal graphs from observed, non-interventional data, in the possible presence of latent confounders and selection bias. We rely on the causal Markov and faithfulness assumptions and recover the equivalence class of the underlying causal graph by performing a series of conditional independence (CI) tests between observed variables. We propose a single step that is applied iteratively, such that the independence and causal relations entailed from the resulting graph, after any iteration, is correct and becomes more informative with successive iteration. Essentially, we tie the size of the CI condition set to its distance from the tested nodes on the resulting graph. Each iteration refines the skeleton and orientation by performing CI tests having condition sets that are larger than in the preceding iteration. In an iteration, condition sets of CI tests are constructed from nodes that are within a specified search distance, and the sizes of these condition sets is equal to this search distance. The algorithm then iteratively increases the search distance along with the condition set sizes. Thus, each iteration refines a graph, that was recovered by previous iterations having smaller condition sets---having a higher statistical power. We demonstrate that our algorithm requires significantly fewer CI tests and smaller condition sets compared to the FCI algorithm. This is evident for both recovering the true underlying graph using a perfect CI oracle, and accurately estimating the graph using limited observed data.
\end{abstract}

\section{Introduction}
Causality plays an important role in many sciences, such as social sciences, epidemiology, and finance. \citep{pearl2010introduction,spirtes2010introduction}.
Understanding the underlying mechanisms is crucial for tasks such as explaining a phenomenon, predicting, and decision making.
\citet{pearl2009causality} provided a machinery for automating the process of answering interventional and (retrospective) counterfactual queries even when only observed data is available, and determining if a query cannot be answered given the available data type (identifiability). This requires a knowledge about the true underlying causal structure;
however, in many real-world situations, this structure is unknown. There is a large body of literature on recovering causal relations from observed data---causal discovery \citep{spirtes2000, peters2017elements, rohekar2018bayesian, yehezkel2009rai, cooper1992bayesian, chickering2002optimal, shimizu2006linear, hoyer2009nonlinear}, differing in the assumptions upon which they rely.
In this work we assume a directed acyclic graph (DAG) for the underlying causal structure and focus on learning it from observational data. Furthermore, we assume the causal Markov and faithfulness assumptions, and consider recovering the structure by performing a series of conditional independence (CI) tests \citep{spirtes2000}. In this setting the true DAG is statistically indistinguishable from many other DAGs. Moreover, when considering the possible presence of latent confounders and selection bias (no causal sufficiency), the true DAG cannot be recovered. Instead, \citet{richardson2002ancestral} proposed the Maximal ancestral graph (MAG), which represents independence relations among observed variables, and the partial ancestral graph (PAG), which is a Markov equivalence class of MAGs---a set of MAGs that cannot be ruled out given the observed  independence relations. 

Recently, causal identification was demonstrated for PAG models \citep{jaber2018causal, jaber2019causal}, which is a more practical use of these models. That is, using only observed data and no prior knowledge on the underlying causal relations, some identification and causal queries can be answered.

In this paper, we address the problem of learning a PAG (latent confounders and selection bias may be present) such that interrupting the learning process results in a correct PAG, that is, all the entailed independence and causal relations are correct, although it can be less informative. This anytime property is important in many real-world settings where it is desired to recover as many causal relations as possible under limited available data and compute power.

\section{Related Work}

Causal discovery in the potential presence of latent confounders and selection bias requires placing additional assumptions. In this paper we assume the causal Markov assumption \citep{pearl2009causality}, faithfulness \citep{spirtes2000}, and a DAG structure for the underlying causal relations. In this setting, several causal discovery algorithms have been proposed, FCI \citep{spirtes2000}, RFCI \citep{colombo2012learning}, FCI+ \citep{claassen2013learning}, and GFCI \citep{ogarrio2016hybrid}. Limitations of FCI have been reported previously where it tends to erroneously exclude many edges that are in the true underlying graph, and it requires many conditional independence (CI) tests with large condition sets. The GFCI algorithm employs a greedy score-based approach to improve the accuracy for small data sizes (small sample). However, it requires additional assumptions for justifying the score function that it uses. The RFCI algorithm alleviates computational complexity by avoiding the last stage. This stage requires many CI tests having large condition sets. Although it is sound (outputs correct causal information) it is not complete (some MAGs in the equivalence class can be ruled out given the data).

Similarily to the FCI and FCI+ algorithms, we consider a procedure that is sound and complete in the large sample limit (or when a perfect conditional independence oracle is used).  
However, these algorithms, for a MAG, treat nodes that are m-separated by adjacent nodes differently from nodes that are m-separated by a minimal separating set that includes nodes outside the neighborhood. In contrast, we treat all possible d-separating sets in a similar manner, and employ an iterative procedure, similar to the one used by the PC algorithm \citep{spirtes2000}. This allows our method to be interrupted at any iteration, returning a correct graph, similarly to the anytime FCI algorithm. \citep{spirtes2001anytime}. We call the proposed algorithm iterative casual discovery, ICD.

\section{Anytime Iterative Discovery of Causal Relations}

We will start by describing notations, definitions and assumptions. Then, we will describe the proposed ICD algorithm and prove its correctness. We conclude by discussing its computational complexity.

\subsection{Preliminaries}

A causal DAG, $\cdag$, over nodes $\nodes=\obs\cup\sel\cup\lat$ is denoted by $\cdagset{\obs}{\sel}{\lat}$, where $\obs$, $\sel$, and $\lat$ represent disjoint sets of observed, selection, and latent variables, respectively. We use a ancestral graph \citep{richardson2002ancestral} to model the conditional independence relations among the observed variables $\obs$ in the causal DAG $\cdag$. This class of graphical models is useful since for every causal DAG there exits a unique maximal ancestral graph (MAG). In this setting, our method is aimed at recovering the MAG of the true underlying DAG. 

We assume the causal Markov and faithfulness assumption, and for CI testing we use either a perfect oracle (independence is derived directly from d-separation in the true DAG), or a statistical independence test (e.g., partial correlation) using observed, non-experimental data. In this setting, the MAG cannot be fully recovered and only a Markov equivalence class, represented by a partial ancestral graph (PAG), can be recovered.
In this graph, an invariant edge-mark is denoted by an empty circle `---o'. That is, there exists at least one MAG in the equivalence class that has a tail edge-mark and at least one MAG that has an arrowhead at that edge tip.

\begin{dff}[$\obs$-equivalence \citep{spirtes2000}]\label{dff:o-equiv}
    Two DAGs, $\cdagseti{\obs}{\sel_i}{\lat_i}{i}$ and  $\cdagseti{\obs}{\sel_j}{\lat_j}{j}$ are said to be $\obs$\textbf{-equivalent} if and only if
    $$ \mathbf{X} \indep \mathbf{Y}|(\mathbf{Z}\cup\sel_i) \mathrm{~in~} \cdag_i \iff \mathbf{X} \indep \mathbf{Y}|(\mathbf{Z}\cup\sel_j) \mathrm{~in~} \cdag_j, $$
    for disjoint subsets $\mathbf{X}$, $\mathbf{Y}$, and $\mathbf{Z}$ of $\obs$. 
\end{dff}
That is, the observed d-separation relations in two $\obs$-equivalent DAGs, $\cdag_i$ and $\cdag_j$, are identical. \citet{spirtes2001anytime} defines equivalence considering an $n$-\textbf{oracle} for testing conditional independence. It returns ``dependence'' if the condition set size is larger than $n$, otherwise d-separation is tested and returned. Thus, $n$-$\obs$-\textbf{equivalence} class is defined by using the $n$-\textbf{oracle} instead of d-separation in \dffref{dff:o-equiv}.

In an MAG $\MAG$, node $X$ is in $\DSEPset{A}{B}$ if and only if $X \neq A$ and there is an undirected path between $A$ and $X$ such that every node on the path, except the endpoints, is a collider and is an ancestor of $A$ or $B$ \citep{spirtes2000}. \citet{spirtes2000} defined a super-set of D-Sep that can be determined from the skeleton of the PC algorithm and its identified v-structures. This super-set, denoted $\PDSEPset{A}{B}$ (note the capitalization of `D'), includes all the nodes connected by a path to $A$, where every node on this path, except the end-points, is a collider or part of a triangle, hiding it orientation. 
It is important to note that this definition of $\PDSepop$ assumes that all conditional independence relations $X_i \indep X_j | \boldsymbol{Z}$, such that $\boldsymbol{Z}\in \parents{\{X_i,X_j\}}{\MAG}$ are identified.
A definition of a smaller super-set of $\DSepop$, called $\PDSeppathop$, was given by \citet{colombo2012learning}, where there is a path connecting every node $Z\in\PDSeppathop(A, B)$ to $B$. We take special interest in the path that connects any $Z\in\PDSEPset{A}{B}$ to $A$ and define it as follows.

\begin{dff}[PDS-path]\label{dff:pds-path}
    A possible-D-Sep-path (PDS-path) from $A$ to $Z$ with respect to $B$, denoted $\PDSPath{A}{B}{Z}$, is an undirected path $\langle A,\ldots,Z\rangle$ such that
    \begin{enumerate}
        \item for every sub-path $\langle U,V,W \rangle$ of $\PDSPath{A}{B}{Z}$, $V$ is a collider or $\{U, V, W\}$ forms a triangle, and
        \item $W$ is on a path between $A$ and $B$.
    \end{enumerate}
    
\end{dff}

\subsection{Proposed Algorithm}

We propose a single step that is called iteratively, for recovering the underlying equivalence class, represented by a PAG. Each iteration is parameterized by $r$ ($r\in \{0, \ldots, |\obs|-1\}$), and given a PAG returned by the previous iteration with $r-1$, each pair of connected nodes $A$ and $B$ are tested for independence conditioned on a set $\mathbf{Z}\subset\obs$, such that
\begin{enumerate}
    \item $|\mathbf{Z}|=r$, and
    \item 
    the length of the shortest PDS-path from $Z\in\mathbf{Z}$ to either $A$ or $B$ is less than $r$.
\end{enumerate}

By parameter $r$, we bind the condition set size to its distance from the tested nodes. That is, the condition set size is bounded by the shortest PDS-path length connecting its nodes to the tested nodes ($A$ and $B$). Only nodes that are at most $r$ edges away, on a PDS-path, from $A$ or $B$ are considered.
Then, identified v-structures are oriented and the iteration concludes by repeatedly applying orientation rules until no more edges can be oriented \citep{spirtes2000, zhang2008completeness}. The full procedure is given in \algref{alg:Alg}

For the first iteration, $r=0$ and condition sets of CI tests are empty. In the second iteration $r=1$ and only nodes that are adjacent to the tested nodes are included in the condition set (PDS-path length limit is one edge). Only in succeeding iterations nodes that are not adjacent to the tested nodes may be included.

\begin{algorithm}
\SetKwInput{KwInput}{Input}                
\SetKwInput{KwOutput}{Output}              
\SetKw{Break}{break}
\DontPrintSemicolon
  
  \BlankLine
  
  \KwInput{
    \\\quad $\pag$: Initial PAG over observed variables $\obs$ (default: $\pag$ is fully connected with o---o edges)
    \\\quad $r_0$: Initial condition set size for CI tests (default: $r_0=0$)
    \\\quad $n$: Desired n-representing PAG, largest condition set size to use (default: $n=|\obs|-2$)
    \\\quad $\sindep$: a conditional independence oracle
    }
    
    \BlankLine
    
  \KwOutput{\\\quad A PAG for $n$-$\obs$-equivalence class}

  \SetKwFunction{FMain}{Main}
  \SetKwFunction{FIter}{CDIteration}
  \SetKwFunction{FPDSepRange}{PDSep\_r}

\BlankLine\BlankLine

  \SetKwProg{Fn}{Function}{:}{}
  \Fn{\FMain{$\pag$, $r_0$, $n$}}{
        \For{$r$ ~in~ $\{r_0, \ldots, n\}$}{
            $(\pag, done) \leftarrow$ \FIter($\pag$, $r$)\;
            \If{$done=\mathrm{True}$}{\Break}
        }
        \KwRet $\pag$\;
  }

\BlankLine\BlankLine

  \SetKwProg{Fn}{Function}{:}{\KwRet}
  \Fn{\FIter{$\pag$, $r$}}{
        $done \leftarrow \text{True}$\;
        \For{$(X\text{\textasteriskcentered---\textasteriskcentered} Y)$ ~in~ $\Edges(\pag)$}{
            $\{\mathbf{Z}_i\}_{i} \leftarrow$ \FPDSepRange($X$, $Y$, $\pag$, r)\Comment*{returns sets of $r$ nodes in range $r$}

            \If{$\{\mathbf{Z}_i\}_{i} \neq \emptyset$}{
                  $done \leftarrow \text{False}$\; 

                \For{$\mathbf{Z'}$ ~in~ $\{\mathbf{Z}_i\}_i$}{
                
                    \If{$\sindep(X, Y| \mathbf{Z'})$}{
                        remove edge $(X\text{\textasteriskcentered---\textasteriskcentered} Y)$ from $\pag$\;
                        \Break\;
                    }
                }
            }
        }
        orient edges in $\pag$\;

        \KwRet $(\pag, done)$\;
  }
  \caption{Anytime Iterative Discovery of Causal Relations (ICD algorithm)}
  \label{alg:Alg}
\end{algorithm}

It is important to note that the result of \FPDSepRange($X$, $Y$, $\pag$, r), in \algref{alg:Alg}-line 10, is an ordered set of possible separating sets $\{\mathbf{Z}_i\}_i$, where each $\mathbf{Z}_i$ contains exactly $r$ nodes. The order in which these sets are assigned to $\mathbf{Z}'$ in \algref{alg:Alg} is crucial for reducing the total number of CI test that are performed. A trivial example for an inefficient order is having a condition set consisting of only parents of $A$ or $B$ last in $\{\mathbf{Z}_i\}_i$. We propose ordering this set according to the average of the shortest PDS-path lengths connecting each node in $\mathbf{Z}_i$. That is, for a $\mathbf{Z}\subset\PDSeppathop(X,Y)$, the following value is calculated,

\begin{equation}
    \hat{d}_X (\mathbf{Z}) = \frac{1}{|\mathbf{Z}|} \sum_{W\in\mathbf{Z}} \min(|\PDSPath{X}{Y}{W}|).
\end{equation}

Then, the possible separating sets, $\mathbf{Z}_i$ are ordered according to this value. Nevertheless, note that the correctness of the algorithm is invariant to this order.

\subsection{Correctness}

Our proof relies on the following Lemmas.

\begin{lem}\label{lem:pds-path}
    Let $\pag$ be a PAG n-representing $\cdagset{\obs}{\sel}{\lat}$. If $A \indep B | [\mathbf{Z}]$, for $A,B\in\obs$ and $\mathbf{Z} \subset \obs$, then any node $V\in\obs$ on the shortest PDS-path in $\pag$ from either $A$ or $B$ to $Z \in \mathbf{Z}$ is also in $\mathbf{Z}$.
\end{lem}

In words: the shortest path connecting a node to $A$ or $B$ consists of a subset of nodes of the minimal separating set (denoted by square brackets). 

\begin{proof}
It is easy to see that a PDS-Path $\PDSPath{A}{B}{Z}$ in a PAG resulting by any iteration of our algorithm has the following properties: 
\begin{enumerate}
    \item every sub-path $\langle A,\ldots, V\rangle$ of $\PDSPath{A}{B}{Z}$ is a PDS-Path $\PDSPath{A}{B}{V}$ (recursion)
    \item every node on $\PDSPath{A}{B}{Z}$ is in $\PDSeppathop(A,B)$
\end{enumerate}

Thus, a condition set that includes node $V$ should also include all the nodes on the shortest PDS-path. 
\end{proof}

\begin{lem}\label{lem:cisize}
    Let $\pag$ be a PAG n-representing $\cdagset{\obs}{\sel}{\lat}$. The number nodes on the shortest PDS-path connecting every node in $\mathbf{Z}$ to $A$ or $B$ is less than n.
\end{lem}

\begin{proof}
From \lemref{lem:pds-path} it directly follows that the condition set size is at least as the length of the shortest PDS-path.
\end{proof}

\begin{thm}
Let $\MAG(\obs)$ be a MAG for a causal DAG $\cdagset{\obs}{\lat}{\sel}$ and let $\mathrm{Ind}$ be a conditional independence oracle that returns d-separation relation for $\obs$ in $\cdag$. If \algref{alg:Alg} is called with $\mathrm{Ind}$ as the conditional independence oracle, then every casual feature entailed by the returned PAG $\pag$ is also entailed by $\MAG$.
\end{thm}

\begin{proof}
\citet{spirtes2001anytime} proposed an anytime algorithm for causal discovery. He proved that a graph constructed only from CI tests having a condition set up to $r$ is correct, although less informative. The skeleton of the graph includes edges only between those nodes that cannot be d-separated conditioned on any subset (not just the parents) having a size bounded by $r$. He proved that relying on such skeleton ensures that any edge-mark that is oriented by identifying a v-structure and by applying orientation rules is invariant (present in any MAG in the equivalence class). It follows from \lemref{lem:cisize} that after iteration $r$ of \algref{alg:Alg}, every edge between nodes that are d-separated by a minimal subset bounded by size $r$ was removed. Thus, from the proof anytime FCI, after refining the skeleton at iteration $r$, any edge-mark that is oriented by identifying a v-structure and by applying orientation rules is invariant is correct.
\end{proof}

\subsection{Complexity and Efficiency Analysis}

The proposed algorithm consists of two main mechanisms for improving efficiency. First, It gradually performs CI tests of increasing condition set sizes, similarly to the PC algorithm. This is done by including all subsets of $\PDSepop$ having a specific size. Secondly, after iteration $r$ all edges are removed between nodes that are m-separated in the underlying MAG by a minimal set of size at most $r$. Thus, we can complete orienting the graph (using v-structures and orientation rules). Moreover, the definition of $\PDSepop$ relies on a strong assumption that the only edges that were removed, are between nodes that are m-separated given their parents in the MAG. In our case, edges between nodes that are m-separated given any minimal subset of size $r$ are removed. This allows considering super-sets of $\DSepop$ that are smaller than $\PDSepop$, resulting in fewer CI tests. We expect this to be more dominant for CI tests with large condition sets.

Finally, we note that the computational complexity of the algorithm is $O(n^{2k})$, where $n$ is the number of nodes and $k$ is the maximal in-degree. This is because for every pair of nodes $A$ and $B$, the number of CI tests is bounded by the number of subsets of $\PDSEPset{A}{B}$ and $\PDSEPset{A}{B}$. That is, a complexity that is polynomial in the number of nodes $n$ given the maximal in-degree $k$.

\section{Experimental Evaluation}

We compare the proposed ICD algorithm to the prominent FCI algorithm \citep{spirtes2000}---both are anytime, sound, and complete. We empirically compare the complexity, run-time and structural accuracy of the algorithms using two types of CI tests: 1) a perfect, always correct, CI oracle, and 2) a statistical independence test evaluated from finite-size data sets. 
We follow a procedure, similar to the one described by \citet{colombo2012learning}, for creating random DAGs and data sets. Specifically, we generate DAGs having $n\in\{15, 20, 25, 35\}$ nodes with a connectivity factor of $\rho=2$. For sampling a random DAG, having $n$ variables and connectivity factor $\rho$, an adjacency matrix $\boldsymbol{A}$ for variables $\obs\cup\lat$ of DAG $\cdagset{\obs}{\lat}{\sel=\emptyset}$ is created by independent realization of $\mathrm{Bernoulli}(\nicefrac{\rho}{(n-1)})$ in the upper triangle. If the resulting DAG is unconnected, we repeat until a connected DAG is sampled.
Then, a weight matrix $\boldsymbol{W}$ for the graph edges is created by sampling from $\mathrm{Uniform}([-2, -0.5]\cup[0.5, 2])$ for each non-zero element in $\boldsymbol{A}$.
Finally, a statistical model is created by setting conditional probabilities $p(X_i|\mathrm{Pa}_{\boldsymbol{A}}(X_i))=\boldsymbol{W}_{(\cdot,i)}\boldsymbol{A}^{\mathrm{T}}_{(\cdot,i)}+\epsilon_i$, where $\epsilon_i\sim\mathcal{N}(0,1)$.
For each DAG, we sample half of the parentless nodes that have at least two children and set them to be latent set ($\lat$). The remaining nodes are the observed set ($\obs$), and selection bias is not simulated ($\sel=\emptyset$).

\subsection{Learning the True PAG using a Perfect Independence Oracle}

The proposed ICD algorithm is proved to recover the PAG that corresponds to the true underlying DAG when using a perfect oracle for the CI tests. A perfect oracle is implemented such that it returns d-separation relations in the true DAG. In this experiment, we measure the number of unique\footnote{We implemented a caching mechanism to ensure that the same CI test is not evaluated and counted more than once.} CI tests and their condition set sizes that are required for recovering the true PAG, and compare it to those that are required by the FCI algorithm.

First, we evaluate the total number of CI tests and run-times for different graph sizes. We randomly generate 25 dags for each of the graphs sizes: 15 nodes, 20 nodes, 25 nodes, and 35 nodes. From \figref{fig:total-ci} (a), we find that compared to FCI, the total number of CI tests required by the proposed method increases more slowly with the graph size the the FCI algorithm (note the logarithmic $Y$-axis). From \figref{fig:total-ci} (b) it is evident that the proposed algorithm requires fewer CI tests compared to FCI, for all the 100 tested graphs, and that this advantage of ICD is more evident for graphs that require a larger number of CI tests. From \figref{fig:total-ci} (c) we note that the ratio between the run-times of FCI and ICD increases with the graph size.

Next, we analyze the condition set sizes of the required CI tests. In many CI tests for real-world applications, the statistical power decreases and the computational complexity grows exponentially with the condition set size. In \figref{fig:ci-cond-size} for each tested graph size, we depict the average number of required CI tests per condition set size. It is evident that the saving in number of CI tests, compared to FCI, increases with the graph size.

Lastly, we provide a visualization for the distribution of the ratio between the number of CI tests of FCI and ICD \figref{fig:ci-vis}. It is evident that the ratio increases with the graph size. Moreover, this distribution is skewed towards larger condition sets and that this skew increases with the graph size.

Finally, we note that in our implementation, the significant reduction in the number of CI tests by ICD compared to FCI, led to a proportional reduction in run-time.

\begin{figure}
  \centering
  \includegraphics[width=0.32\textwidth]{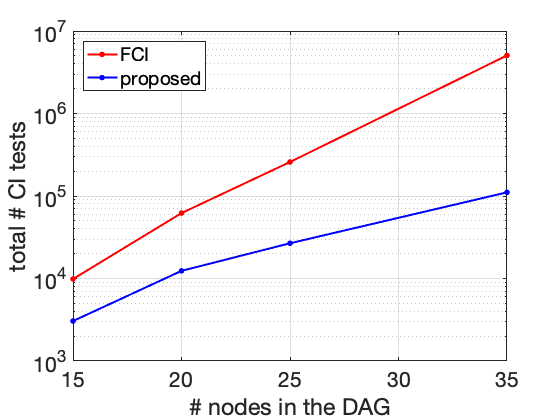}
  \includegraphics[width=0.32\textwidth]{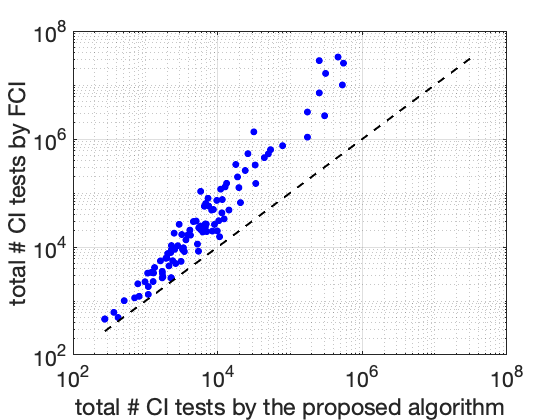}
  \includegraphics[width=0.32\textwidth]{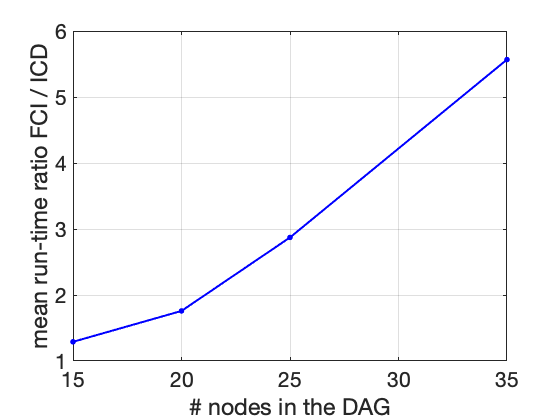}
  \\
  (a)\quad\quad\quad\quad\quad\quad\quad\quad\quad\quad\quad\quad(b)\quad\quad\quad\quad\quad\quad\quad\quad\quad\quad\quad\quad(c)
  \\
  \caption{Total number of CI tests and run-time. (a) Average total number of CI tests as a function of the graph size ($Y$-axis is logarithmic). (b) A scatter plot using all DAGs in the experiment (both $X$ and $Y$ axes are logarithmic). For all the 100 tested DAGs, the proposed ICD algorithm requires fewer CI tests for recovering the true underlying PAG compared to FCI.(c) Mean run-times ratio. Run-time of FCI is divided by the ICD run-time (both are implemented and run on the same platform). \label{fig:total-ci}}
\end{figure}

\begin{figure}
  \centering
  \includegraphics[width=0.245\textwidth]{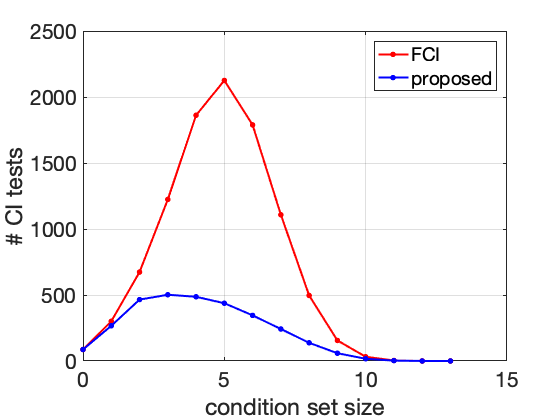}
  \includegraphics[width=0.245\textwidth]{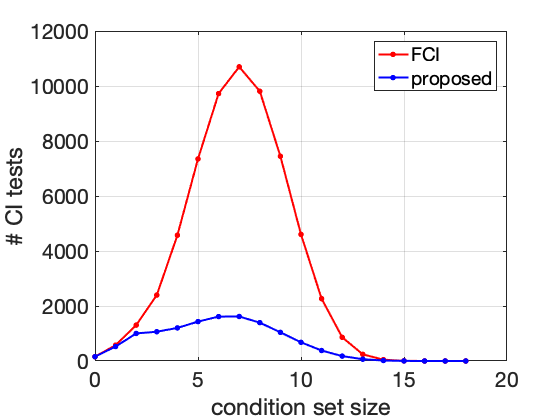}
  \includegraphics[width=0.245\textwidth]{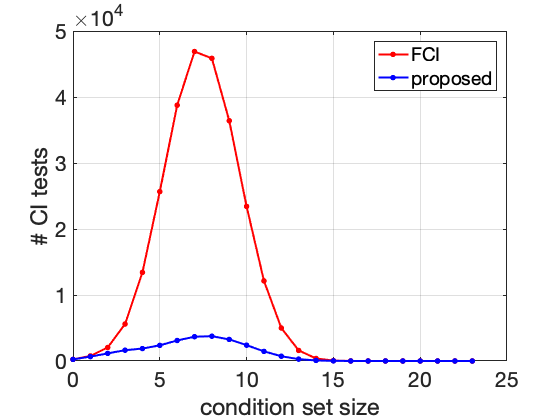}
  \includegraphics[width=0.245\textwidth]{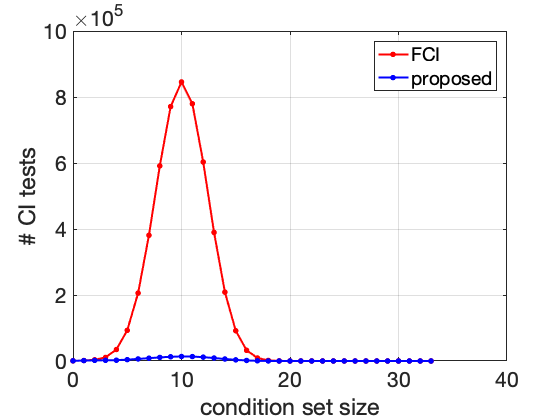}
  \\
  (a)\quad\quad\quad\quad\quad\quad\quad\quad\quad(b)\quad\quad\quad\quad\quad\quad\quad\quad\quad(c)\quad\quad\quad\quad\quad\quad\quad\quad\quad(d)
  \\
  \caption{Average number of CI tests per condition set size for different graph sizes. (a) 15 nodes, (b) 20 nodes, (c) 25 nodes, (d) 35 nodes. The proposed method provides a greater saving in CI tests with large condition sets for larger graphs.\label{fig:ci-cond-size}}
\end{figure}

\begin{figure}
  \centering
  \includegraphics[width=0.245\textwidth]{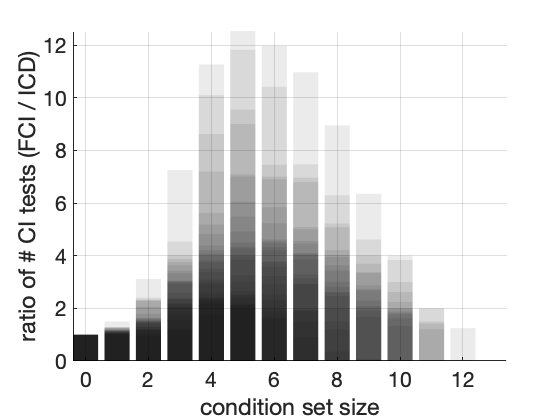}
  \includegraphics[width=0.245\textwidth]{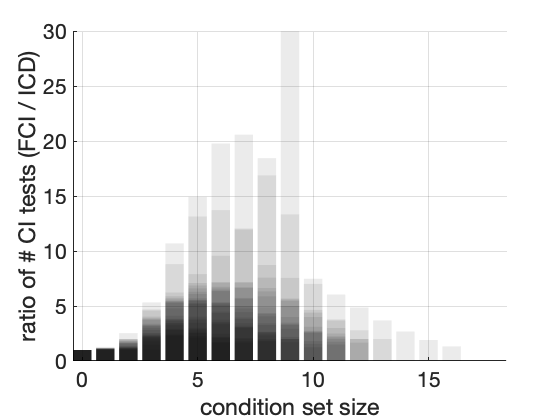}
  \includegraphics[width=0.245\textwidth]{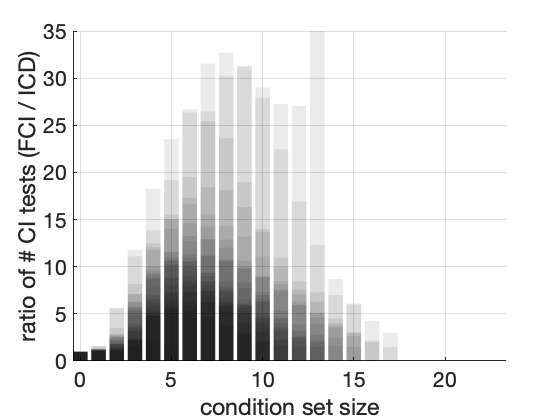}
  \includegraphics[width=0.245\textwidth]{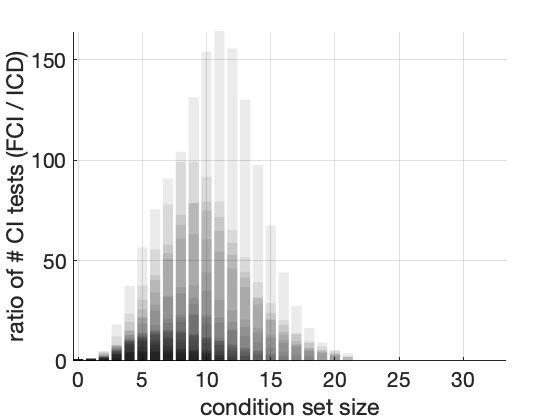}
  \\
  (a)\quad\quad\quad\quad\quad\quad\quad\quad\quad(b)\quad\quad\quad\quad\quad\quad\quad\quad\quad(c)\quad\quad\quad\quad\quad\quad\quad\quad\quad(d)
  \\
  \caption{A visualization of the distribution (shading) of the ratio between the number of CI tests required by FCI and the proposed ICD algorithm. For each DAG in the experiment, a semi-transparent bar-plot is created by dividing the FCI number of CI tests (per condition set size) by the ICD number of CI tests. Greater ratios mean fewer CI tests are required by the proposed method compared to FCI for recovering the true underlying PAG. All plots are layered one on top of the other. Thus, dark areas are proportional to high probability for that condition set size at that ratio. Number of graph nodes: (a) 15 nodes, (b) 20 nodes, (c) 25 nodes, (d) 35 nodes. It can be viewed that the distribution is skewed (the shading decreases more gradually) towards larger condition sets.\label{fig:ci-vis}}
\end{figure}

\subsection{Learning a PAG from Observed Data}

In this experiment, we use the partial-correlation for testing conditional independence (CI) at $\alpha=0.01$, and examine two aspects: accuracy of the learned PAG and  the required number of CI tests. For evaluating partial correlation we use data (samples for the observed variables $\obs$) in various sizes, $\ell \in [10, 1000]$. For each data size, we randomly sample $1000$ DAGs, each having 15 nodes. We learn the PAGs using the proposed ICD algorithm and the FCI algorithm (both are sound and complete).

\textbf{Required Number of CI Tests}. We analyze the total number CI tests required by the proposed algorithm compared to the FCI algorithm. First, using 1000 learned DAGs, we compute the empirical cumulative distribution function $F$ (CDF) for both algorithms on different data sizes. The empirical CDFs for data sizes $100$, $200$, $500$, and $1000$ is depicted in \figref{fig:cdf}. We mark $F(t')=0.9$ with dashed horizontal line and with vertical line we mark $t'$ in $90\%$. That is, in $90\%$ of the tests (learned DAGs) the algorithm required at most $t'$ CI tests. For data sizes $100$, $200$, $500$, and 1000, the ratios between $t'$ for our proposed method and FCI are $1.16$, $1.75$, $2.76$, and $3.75$, respectively. It is evident that the ratio increases rapidly for small data sets and starts to converge to some value (the ratio when using perfect CI oracle). 
We then evaluate statistical significance using the 2-sample Kolmogorov–Smirnov test. The p-values and the statistic are given in \figref{fig:kstest} (a) and (b), respectively. The value of the Kolmogorov–Smirnov statistic for $\ell\geq 40$ is consistently high compared to its value for smaller data sizes. We conclude that for $\ell\geq40$ the CDFs of the number of CI tests is different with statistical significance at $\alpha=0.05$. Finally, In \figref{fig:kstest} (c) we plot the number of CI tests as a function of data set size.

\textbf{Structural Accuracy}. We measure three types of structural errors: (`extra-edges') the number of edges present in the learned structure but are not in the true PAG, (`missing-edges') the number of edges that are in the true PAG but are missing from the learned PAG, and ('wrong edge-marks') the number of edge-marks in the learned graph that are different in the true PAG (for edges that exist in both graphs). Compared to the FCI algorithm and for $100 \leq \ell \leq 1000$, the proposed method had fewer, (up to $10\%$ for 15 nodes, and $18\%$ for 35 nodes), `missing-edges' and a similar a number of 'extra-edges' and 'wrong edge-marks'. We report the ratio in total number of structural errors for graphs having 15 nodes, as a function of the data set size in \figref{fig:kstest} (d). 

\begin{figure}
  \centering
  \includegraphics[width=0.245\textwidth]{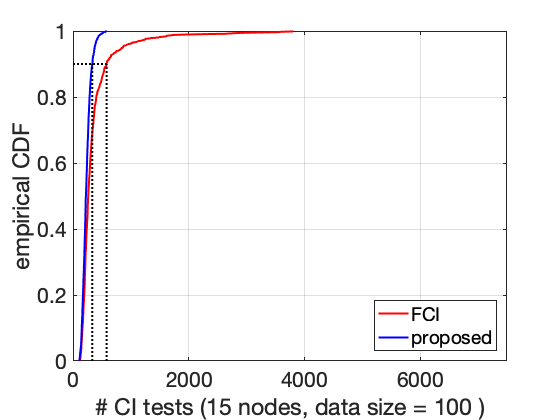}
  \includegraphics[width=0.245\textwidth]{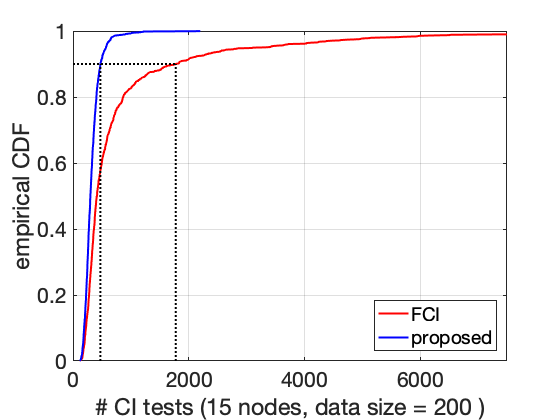}
  \includegraphics[width=0.245\textwidth]{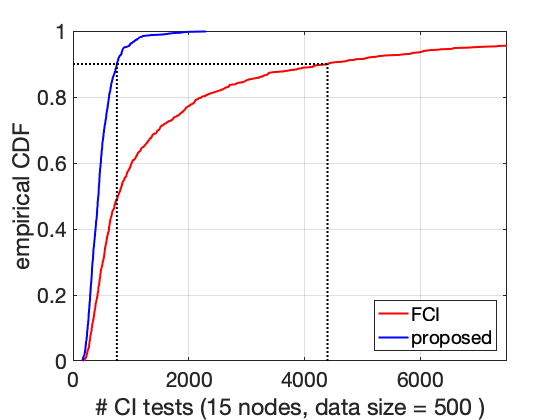}
  \includegraphics[width=0.245\textwidth]{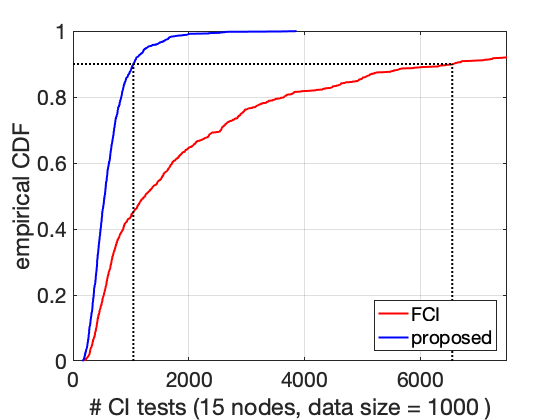}
  \\
  (a)\quad\quad\quad\quad\quad\quad\quad\quad\quad(b)\quad\quad\quad\quad\quad\quad\quad\quad\quad(c)\quad\quad\quad\quad\quad\quad\quad\quad\quad(d)
  \\
  \caption{The empirical cumulative distribution function $F(t)$ (CDF) for the number of CI tests $t$. Evaluated by 1000 randomly generated DAGs. Horizontal line indicates $F(t')=0.9$, vertical lines indicate the corresponding value $t'$. That is, $90\%$ of experiment resulting in number of CI tests fewer than $t'$. Data size used for calculating CI tests: (a) 100, (b) 200, (c) 500, (d) 1000, and the ratios between $t'$ for our proposed method and FCI are 1.16, 1.75, 2.76, and 3.75, respectively.\label{fig:cdf}}
\end{figure}

\begin{figure}
  \centering
  \includegraphics[width=0.245\textwidth]{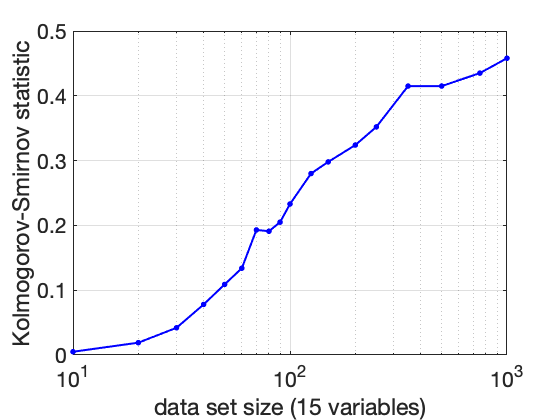}
  \includegraphics[width=0.245\textwidth]{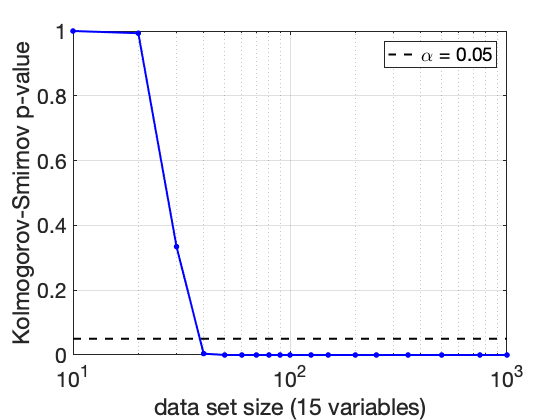}
  \includegraphics[width=0.245\textwidth]{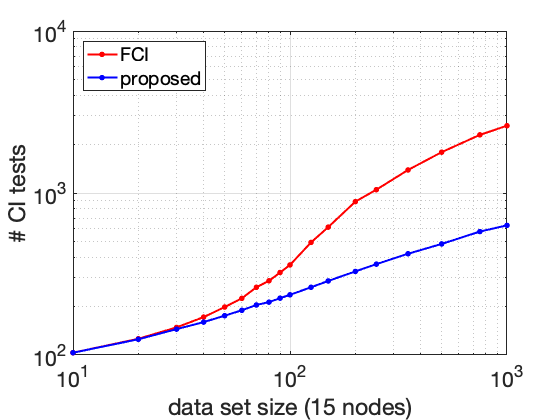}
  \includegraphics[width=0.245\textwidth]{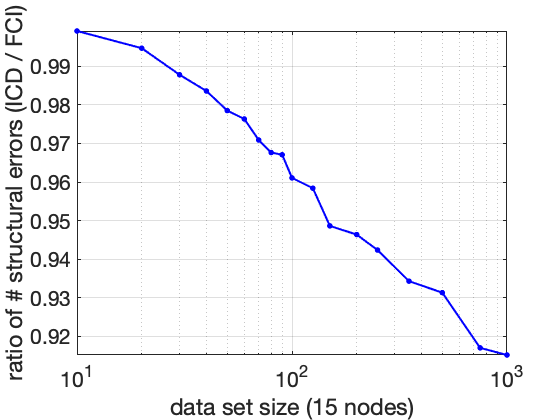}
  \\
  (a)\quad\quad\quad\quad\quad\quad\quad\quad\quad(b)\quad\quad\quad\quad\quad\quad\quad\quad\quad(c)\quad\quad\quad\quad\quad\quad\quad\quad\quad(d)
  \\
  \caption{Structural error and the difference in the number of CI tests as a function of the data size. The FCI and ICD algorithms learned 1000 randomly generated DAGs with 15 nodes. The Kolmogorov-Smirnov test for comparing the difference in the CDFs of the number of CI tests: (a) statistic, and (b) p-value. (c) The mean number of CI tests. (d) The ratio between the number of structural errors of FCI and ICD. \label{fig:kstest}}
\end{figure}

\section{Conclusions}

We proposed an anytime, sound, and complete causal discovery algorithm that consists of a single step that is applied iteratively. In every iteration, the skeleton is sought to be refined and edges are oriented using the same orientation rules used by the FCI algorithm. From the experimental results, the proposed method requires significantly fewer CI tests compared to FCI (along with a small improvement in accuracy), especially for large condition sets. This is reflected in significantly shorter run-times.

One important difference of the proposed ICD algorithm from FCI and its related algorithms is that, right from the outset, it considers nodes for the condition set, that are not in the neighborhood of the tested nodes.
One might suspect that this could result in a high number of CI tests evaluated by the proposed algorithm compared to FCI. However,
the reliance on complete orientation, which is missing in FCI, enables the proposed algorithm to reduce the number of nodes to consider for the condition sets.

We proposed an ordering for the execution of CI tests when testing independence between a pair of nodes. This ordering is proportional to the distance of the condition set from the tested pair. Although the correctness of our algorithm is invariant to this ordering, it is enables further reducing the number of CI tests. In our future work we plan to investigate different criteria for ordering.

\newpage

\bibliography{arXiv_ICD_CausalityNeurIPS}


\end{document}